\documentclass[journal,twoside,web]{ieeecolor}
\usepackage{generic}

\usepackage{amsmath,amsfonts,amssymb}
\usepackage{mathtools}
\usepackage{bbm}
\usepackage{epsfig} %
\usepackage{times} %
\usepackage[dvipsnames]{xcolor}
\usepackage[english]{babel}
\usepackage{graphicx}
\usepackage{subfigure}
\usepackage{nicefrac}
\definecolor{sztakiblue}{RGB}{0,0,0}

\def\argmin{\operatornamewithlimits{arg\,min}}

\newcommand\myeq{\stackrel{\mathclap{\normalfont\tiny\mbox{a.s.}}}{=}}

\newcommand{\PX}{P_{\scriptscriptstyle X}}
\newcommand{\PY}{P_{\scriptscriptstyle Y}}
\newcommand \dd[1]  { \,\textrm d{#1}}

\newcommand{\BX}{\mathbb{X}}
\newcommand{\BY}{\mathbb{Y}}

\newcommand{\BA}{\mathbb{A}}

\newcommand{\BI}{\mathbb{I}}

\newcommand{\CF}{\mathcal{F}}
\newcommand{\CA}{\mathcal{A}}
\newcommand{\CB}{\mathcal{B}}

\newcommand{\CH}{\mathcal{H}}

\newcommand{\CR}{\mathcal{R}}

\def \CG{\mathcal{G}}
\def \CF{\mathcal{F}}

\def \CD{\mathcal{D}}

\newcommand{\RR}{\mathbb{R}}

\newcommand{\NN}{\mathbb{N}}

\newcommand{\EE}{\mathbb{E}}
\newcommand{\PP}{\mathbb{P}}
\newcommand{\Var}{\mbox{Var}}

\def\argmin{\operatornamewithlimits{arg\,min}}        %

\DeclareMathOperator {\sign}{sign}

\newtheorem{theorem}{Theorem}

\newtheorem{definition}{Definition}
\newtheorem{proposition}{Proposition}
\newtheorem{remark}{Remark}

\newcommand{\norm}[1]{\left\lVert#1\right\rVert}

\title{\LARGE \bf
Exact Distribution-Free Hypothesis Tests for the Regression Function\\
of Binary Classification via Conditional Kernel Mean Embeddings
}
\author{Ambrus Tam\'as$^{1}$\; and\;\, Bal\'azs Csan\'ad Cs\'aji$^{1}$, \IEEEmembership{Member, IEEE}%
\thanks{*The research 
was supported by the Ministry of Innovation and Technology of Hungary NRDI Office within the framework of the Artificial Intelligence National Laboratory Program. Prepared with the professional support of the doctoral student scholarship program of the Cooperative Doctoral Program of the Ministry of Innovation and Technology financed  from the National Research, Development and Innovation Fund.
}
\thanks{$^{1}$A.~Tam\'as and B.~Cs.~Cs\'aji are with SZTAKI: Institute for Computer Science and Control, E\"otv\"os Lor\'and Research Network, Budapest, Hungary,\, {\tt\small ambrus.tamas@sztaki.hu},\, {\tt\small csaji@sztaki.hu}}%
}

\begin{document}

\maketitle
\thispagestyle{empty}
\pagestyle{empty}

\begin{abstract}
In this paper we suggest two statistical hypothesis tests for the regression function of binary classification based on conditional kernel mean embeddings. The regression function is a fundamental object in classification as it determines both the Bayes optimal classifier and the misclassification probabilities. A resampling based framework is presented and combined with consistent point estimators of the conditional kernel mean map, in order to construct distribution-free hypothesis tests. These tests are introduced in a flexible manner allowing us to control the exact probability of type I error for any sample size.  We also prove that both proposed techniques are consistent under weak statistical assumptions, i.e., the type II error probabilities pointwise converge to zero.
\end{abstract}

\vspace*{-1mm}
\section{Introduction}

\IEEEPARstart{B}{\lowercase{inary}} classification \cite{devroye2013probabilistic} is a central problem in supervised learning {\color{sztakiblue} with a lot of crucial applications, for example, in quantized system identification, signal processing and fault detection}. Kernel methods \cite{scholkopf2001learning} offer a wide range of tools to draw statistical conclusions 
by embedding datapoints and distributions into a 
(possibly infinite dimensional) 
reproducing kernel Hilbert space (RKHS), where we can take advantage of the geometrical structure. {\color{sztakiblue} These nonparametric methods often outperform the standard parametric 
approaches \cite{pillonetto2014}. A key quantity, for example in model validation, is the conditional distribution of the outputs given the inputs. A promising way to handle such conditional distributions is to apply conditional kernel mean embeddings \cite{song2009a} which are input dependent elements of an RKHS.}
In this paper we introduce distribution-free hypothesis tests for the regression function of binary classification based on these conditional 
{\color{sztakiblue} embeddings.  
Such distribution-free guarantees are of high importance, since our knowledge on the underlying distributions is often limited.}

Let $(\BX, \mathcal{X})$ be a measurable input space, where $\mathcal{X}$ is a $\sigma$-field on $\BX$, and let $\BY =\{-1,1\}$ 
be the output space. In binary classification we are given an independent and identically distributed (i.i.d.) sample $\{(X_i,Y_i)\}_{i=1}^n$ from an unknown distribution $P = P_{\scriptscriptstyle{X,Y}}$ on $\mathcal{X}\otimes \mathcal{Y}$.  Measurable $\BX \to \BY$ functions are called classifiers. Let $\mathbf{L}: \BY \times \BY \to \RR^{+}$ be a nonnegative measurable loss function. In this paper we restrict our attention to the archetypical $0/1$-loss given by the indicator $\mathbf{L}(y_1,y_2) \doteq \BI(y_1 \neq y_2)$ for $y_1$, $y_2  \in \BY$. In general, our aim is to minimize the Bayes risk, which is $R(\phi)\doteq \EE\, [\,\mathbf{L}(\phi(X),Y)\,]$ for classifier $\phi$, i.e., the expected loss. It is known that for the $0/1$-loss, the Bayes risk is the misclassification probability $R(\phi) = \PP(\, \phi(X) \neq Y\,)$ and a risk minimizer ($\PX$-a.e.) equals to the sign of the regression function $f_*(x) \doteq \EE\,[\,Y\, |\, X=x\,]$, i.e., classifier\footnote{Let the $\sign$ function be defined as $\sign(x) = \BI(x\geq 0) - \BI(x <0)$.}  $\phi_*(x) = \sign\hspace{-0.5mm}\big(f_*(x)\big)$ reaches the optimal risk
\cite[Theorem 2.1]{devroye2013probabilistic}. It can also be proved that the conditional distribution of $Y$ given $X$ is encoded in $f_*$ for binary outputs. 

{\color{sztakiblue}One of the main challenges} in statistical learning is that distribution $P$ is unknown, therefore the true risk cannot be directly minimized, only through empirical estimates \cite{Vapnik1998}.
Vapnik's theory 
quantifies the rate of convergence for several 
approaches {\color{sztakiblue}(empirical and structural risk minimization),} but these bounds are usually conservative for small samples. The literature is rich in efficient point estimates, but there is a high demand for 
{\color{sztakiblue} distribution-free uncertainty quantification}.

It is well-known that hypothesis tests {\color{sztakiblue} are closely related to confidence regions.} Distribution-free confidence regions for classification 
received considerable interest, for example, Sadinle {\em et al.} suggested set-valued estimates with guaranteed coverage confidence \cite{sadinle2019least},
Barber studied the limitations of such distribution-free region estimation methods \cite{barber2020distribution}, while Gupta {\em et al.} analyzed score based classifiers and the connection of calibration, confidence intervals and prediction sets \cite{gupta2020distribution}.

{\color{sztakiblue}Our main contribution is that, building on the distribution-free 
resampling framework of \cite{csajitamas2019} which was motivated by finite-sample system identification methods 
\cite{care2017finite}, we suggest conditional kernel mean embeddings based ranking functions to construct hypothesis tests for the regression function of binary classification. Our tests have exact non-asymptotic guarantees for the probability of type I error and have strong asymptotic guarantees regarding the type II error probabilities.}
\smallskip
\vspace*{-1mm}
\section{Reproducing Kernel Hilbert Spaces}

\subsection{Real-Valued Reproducing Kernel Hilbert Spaces}
Let $k:\BX  \times \BX \to \RR$ be {\color{sztakiblue} a symmetric and positive-definite kernel,}
i.e., for all $n \in \NN$, $x_1, \dots, x_n \in \BX, a_1,\dots,a_n \in \RR$:
\begin{equation}
\begin{aligned}
\!\!\!\sum_{i,j=1}^n k(x_i,x_j)\,a_i \,a_j \geq 0.
\end{aligned}
\end{equation}
Equivalently, kernel {\color{sztakiblue} (or Gram)} matrix $K \in \RR^{n\times n}$, where $K_{i,j}\doteq k(x_i,x_j)$ for all $i$, $j \in [n] \doteq \{1,\dots, n\}$, is required to be positive semidefinite. Let $\CF$ denote the corresponding reproducing kernel Hilbert space containing $\BX \to \RR$ functions, see \cite{scholkopf2001learning},
where $k_x(\cdot)=k(\cdot,x) \in \CF$ and the reproducing property,
$f(x) = \langle f,k(\cdot,x)\rangle_{\CF}$,
holds for all $x \in \BX$ and $f \in \CF$. {\color{sztakiblue} Let $l:\BY \times \BY \to \RR$ denote a symmetric and positive-definite kernel and let $\CG$ be the corresponding RKHS.}
%
%
\vspace*{-2mm}
\subsection{Vector-Valued Reproducing Kernel Hilbert Spaces}

The definition of conditional kernel  mean embeddings \cite{song2009a} requires a generalization of real-valued RKHSs \cite{michelli2005, grunewalder2012}.

\begin{definition}
Let $\CH$ be a Hilbert space of $\BX \to \CG$ type functions with inner product $\langle \cdot, \cdot\rangle_\CH$, where $\CG$ is a Hilbert space. 
$\CH$ is a {\em vector-valued RKHS} {\color{sztakiblue}if for all $x \in \BX
	$ and $g \in \CG$ the linear functional (on $\CH$) $h \mapsto \langle g, h(x) \rangle_\CG$ is bounded.
}
\end{definition}

{\color{sztakiblue}Then} by the Riesz representation theorem, for {\color{sztakiblue}all} $(g,x)\in \CG \times \BX$ there exists a unique $\tilde{h} \in \CH$ for which 
{\color{sztakiblue} $\langle g, h(x) \rangle_\CG = \langle \tilde{h}, h \rangle_\CH$}.
Let $\Gamma_x$ be a $\CG \to \CH$ 
operator defined as 
$\Gamma_x g = \tilde{h}$. The notation is justified because $\Gamma_x$ is linear. Further, let $\mathcal{L}(\CG)$ denote the bounded linear operators on $\CG$ and let $\Gamma: \BX \times \BX \to \mathcal{L}(\CG)$ be defined as
{\color{sztakiblue}$\Gamma(x_1,x_2)g \doteq (\Gamma_{x_2} g)(x_1) \in \CG.$}
We will use the following result 
\cite[Proposition 2.1]{michelli2005}:
\begin{proposition}
{\color{sztakiblue} Operator $\Gamma$ satisfies 
for all $x_1$, $x_2 \in \BX$:}
\begin{enumerate}
 {\color{sztakiblue} \item 
 $\forall g_1, g_2 \in \CG:$
$\langle g_1 , \Gamma(x_1,x_2)g_2 \rangle_\CG = \langle \Gamma_{x_1} g_1, \Gamma_{x_2} g_2 \rangle_\CG$.}\\[-3mm]
 \item $\Gamma(x_1,x_2) \in \mathcal{L}(\CG)$, $\Gamma(x_1, x_2) = \Gamma^*(x_2,x_1)$, and for all $x \in \BX$ operator $\Gamma(x,x)$ is positive.\\[-3mm]
 \item For all $n \in \NN$, $\{ (x_i)\}_{i=1}^n \subseteq \BX$ and $\{ (g_j)\}_{j=1}^n \subseteq \CG$:
 \vspace{-1mm}
 \begin{equation}
 \sum_{i,j=1}^n \langle g_i, \Gamma (x_i,x_j) g_j \rangle_\CG \geq 0.
 \vspace{2mm}
 \end{equation}
\end{enumerate}
\end{proposition}
When properties {\em $1)-3)$} hold, we call $\Gamma$ a {\em vector-valued reproducing kernel}. Similarly to the classical Moore-Aronszjan theorem \cite[Theorem 3]{berlinet2004reproducing}, for any kernel $\Gamma$, there uniquely exists (up to isometry) a vector-valued RKHS, having  $\Gamma$ as its reproducing kernel \cite[Theorem 2.1]{michelli2005}.

\vspace*{-1mm}
\section{Kernel Mean Embeddings}

\subsection{Kernel Means of Distributions}
{\color{sztakiblue}
Kernel functions with a fixed argument are feature maps, i.e., they represent input points from $\BX$ in Hilbert space $\CF$ by mapping $x \mapsto k(\cdot,x)$.
Let $X$ be a random variable with distribution $P_{\scriptscriptstyle{X}}$, then $k(\cdot,X)$ is a random element in $\CF$. The kernel mean embedding of distribution $P_{\scriptscriptstyle{X}}$ is defined as $\mu_{\scriptscriptstyle{X}} \doteq \EE [\,k(\cdot, X)\,]$, where the integral is a Bochner integral \cite{hytonen2016analysis}. 

It can be proved that if kernel $k$ is measurable as well as
$\EE \big[\,\sqrt{ k(X,X)}\,\big] < \infty$
holds, then the kernel mean embedding of $\PX$ exists and it is the representer of the bounded, linear expectation functional w.r.t.\ $X$, therefore $\mu_{\scriptscriptstyle{X}} \in \CF$ and we have $\langle f, \mu_{\scriptscriptstyle{X}}\rangle_\CF = \EE\big[\, f(X)\,\big]$ for all $f \in \CF$ \cite{smola2007hilbert}.
Similarly, for variable $Y$ let $\mu_{\scriptscriptstyle{Y}}$ be the kernel mean embedding of $\PY$.}

\vspace*{-2mm}
\subsection{Conditional Kernel Mean Embeddings}
If the kernel mean embedding of $P_{\scriptscriptstyle{Y}}$ exists, then $l_Y \doteq l(\circ, Y) \in L_1(\Omega, \CA,\PP;\CG)$, {\color{sztakiblue} that is $l_Y$ is a Bochner integrable $\CG$-valued random element}, hence for all $\CB \subseteq \CA$ the conditional expected value can be defined. Let $\CB \doteq \sigma(X)$ be the $\sigma$-field generated by random element $X$, then the conditional kernel mean embedding of $P_{Y|X}$ in RKHS $\CG$ is defined as
\begin{equation}
\begin{aligned}\label{eq:cond_mean}
&\mu_{{\scriptscriptstyle{Y|X}}} = {\color{sztakiblue} \mu_*(X)} \doteq \EE \big[\, l(\circ, Y) \,|\, X\, \big],
\end{aligned}
\end{equation}
see \cite{ParMua20}, where {\color{sztakiblue}$\mu_*$} is a $\PX$-a.e. defined (measurable) conditional kernel mean map.
It is easy to see that for all $g \in \CG$ 
\begin{equation}\label{eq:30_inner}
\EE\big[\, g(Y) \,|\,X\,\big] \myeq \langle\, g, \EE[\, l(\circ, Y)\,|\,X\,]\,\rangle_\CG,
\end{equation}
showing that this approach is equivalent to the definition in \cite{grunewalder2012}. We note that the original paper \cite{song2009a} introduced conditional mean embeddings as $\CF \to \CG$ type operators. The presented approach is more natural and has theoretical advantages as its existence and uniqueness is usually ensured.

\vspace*{-2mm}
\subsection{Empirical Estimates of Conditional Kernel Mean Maps}

The advantage of dealing with kernel means instead of the distributions is that we can use the structure of the Hilbert space. In statistical learning, the underlying distributions are unknown, thus their kernel mean embeddings are needed to be estimated. A typical assumption for classification is that:
\begin{itemize}
\item[A0] Sample $\CD_0\! \doteq\! \{(X_i,Y_i)\}_{i=1}^n$ is i.i.d.\ with distribution $P$.
\end{itemize}
%

%
The empirical estimation of conditional kernel mean map {\color{sztakiblue}$\mu_*$} $: \BX \to \CG$ is challenging in general, because its dependence on $x \in \BX$ can be {\color{sztakiblue}complex}. 
The standard approach defines estimator $\widehat{\mu}$ as a regularized empirical risk minimizer in a vector-valued RKHS, see \cite{grunewalder2012}, which is equivalent to the originally proposed operator estimates in \cite{song2009a}. %

By \eqref{eq:30_inner} it is intuitive to
estimate {\color{sztakiblue} $\mu_*$} with a minimizer of the following objective 
over some space $\CH$ 
\cite[Equation 5]{grunewalder2012}:
\begin{equation}
\mathcal{E}({\color{sztakiblue} \mu}) = \sup_{\norm{g}_\CG \leq 1} \EE\big[\,\big|\,E[ g(Y) \,|\,X]- \langle g, {\color{sztakiblue} \mu(X)}  \rangle_\CG \,|^2\,\big].
\end{equation}
Since $\EE[\, g(Y) \,|\,X\,]$ is not observable, the authors of \cite{grunewalder2012} have introduced the following surrogate loss function:
\begin{equation}
\mathcal{E}_s({\color{sztakiblue} \mu})\doteq \EE\Big[\,\norm{l(\circ,Y)- {\color{sztakiblue} \mu(X)}}_\CG^2\Big].
\end{equation}
{\color{sztakiblue}It can be shown \cite{grunewalder2012} that $\mathcal{E}(\mu) \leq \mathcal{E}_s(\mu)$ for all $\mu: \BX \to  \CG$, moreover under suitable conditions \cite[Theorem 3.1]{grunewalder2012} the minimizer of $\mathcal{E}(\mu)$ $\PX$-a.s. equals to the minimizer of $\mathcal{E}_s(\mu)$, hence the surrogate version can be used.} The main advantage of $\mathcal{E}_s$ is that it can be estimated empirically as:
\vspace{-1mm}
\begin{equation}
\label{eq:surrogate}
\widehat{\mathcal{E}}_s({\color{sztakiblue} \mu})\doteq \frac{1}{n}\sum_{i=1}^n\,\norm{l(\circ,Y_i)- {\color{sztakiblue} \mu(X_i)}}_\CG^2.
\vspace{-0.5mm}
\end{equation}
To make the {\color{sztakiblue} problem} tractable, {\color{sztakiblue} we minimize \eqref{eq:surrogate} over a vector-valued RKHS, $\CH$.} There {\color{sztakiblue} are} several choices for $\CH$. 
An intuitive approach is to use the space induced by kernel 
$\Gamma(x_1,x_2) \doteq k(x_1,x_2) \mathbf{Id}_\CG,$ {\color{sztakiblue}where $x_1$, $x_2 \in \BX$ and $\mathbf{Id}_\CG$ is the identity map on 
$\CG$}. %
Henceforth, we will focus on this kernel, as it leads to the same estimator as the one originally proposed 
in \cite{song2009a}. A regularization term is also used to prevent overfitting and {\color{sztakiblue} to ensure well-posedness, hence, the estimator is defined as}
\begin{equation}
\widehat{\mu}=\widehat{\mu}_\CD = \widehat{\mu}_{n,\lambda} \doteq \argmin_{{\color{sztakiblue} \mu}  \in \CH}\widehat{\mathcal{E}}_\lambda({\color{sztakiblue} \mu}) 
\vspace{-1mm}
\end{equation}
where $\widehat{\mathcal{E}}_\lambda({\color{sztakiblue} \mu} ) \doteq \big[\widehat{\mathcal{E}}_s({\color{sztakiblue} \mu}) + \nicefrac{\lambda}{n} \norm{{\color{sztakiblue} \mu}}_\CH^2\big]$.
An explicit form of $\widehat{\mu}$ can be given by \cite[Theorem 4.1]{michelli2005} (cf.\ representer theorem):
\begin{theorem}\label{theorem:michelli}
If $\widehat{\mu}$ minimizes $\mathcal{E}_\lambda$ in $\CH$, then it is unique and admits the form  of 
$\widehat{\mu} = \sum_{i=1}^n \Gamma_{X_i}c_i,$
where coefficients $\{(c_i)\}_{i=1}^n$, $c_i \in \CG$ for $i \in [n]$, are the unique solution of
\vspace{-1mm}
\begin{equation*}
\sum_{j=1}^n \big( \, \Gamma(X_i,X_j) + \lambda \BI (i=j) \mathbf{Id}_\CG\,\big)c_j = l(\circ,Y_i)\; \;\; \text{for}\;\;i \in [n].
\vspace{2mm}
\end{equation*}
\end{theorem}
By Theorem \ref{theorem:michelli} we have: $c_i = \sum_{j=1} W_{i,j} l(\circ, Y_j)$ for $i\in [n]$,
with $W = (K + \lambda \mathbf{I})^{-1}${\color{sztakiblue}, where $\mathbf{I}$ is the identity matrix.}

\section{Distribution-Free Hypothesis Tests}\label{sec-hyp}

For binary classification, one of the most intuitive kernels on the output space is $l(y_1,y_2) \doteq \BI(y_1 =y_2)$ for $y_1$, $y_2 \in \BY$,
which is called the na\"ive kernel. It is easy to prove that $l$ is symmetric and positive definite. Besides, we can describe {\color{sztakiblue} its induced} RKHS $\CG$ as $\{ a_1 \cdot l(\circ,1) + a_2 \cdot l(\circ,-1)\,|\,a_1, a_2 \in \RR \}$. Hereafter, $l$ will denote this kernel for the output space.
\vspace*{-2mm}
\subsection{Resampling Framework}
 
We consider the following hypotheses:
\begin{equation}
\begin{aligned}
&H_0: \;\; f_* = f \;\;\;(P_{\scriptscriptstyle{X}} \text{-a.s.})\\[1mm]
&H_1: \;\; \neg\, H_0
\end{aligned}
\end{equation}
for a given candidate regression function $f$, where $\neg\, H_0$ denotes the negation of $H_0$. For the sake of simplicity, we will use the slightly inaccurate notation $f_* \neq f$ for $H_1$, which refers to inequality in the $L_2(P_{\scriptscriptstyle{X}})$-sense. To avoid misunderstandings, we will call $f$ the ``candidate'' regression function and $f_*$ the ``true'' regression function.

One of our main observations is that in binary classification the regression function determines the conditional distribution of $Y$ given $X$, i.e., by \cite[Theorem 2.1]{devroye2013probabilistic} we have
\begin{equation}
\begin{aligned}
\PP_{*}(Y =1 \,|\,X) = \frac{f_*(X) +1}{2} = p_*(X).
\end{aligned}
\end{equation}
Notation $\PP_{*}$ is introduced to emphasize the dependence of the conditional distribution on $f_*$. Similarly, candidate function $f$ can be used to determine a conditional distribution given $X$. Let $\bar{Y}$ be such that $\PP_{f}(\bar{Y} =1 \,|\,X) = (f(X) +1)\,/\,{2}= p(X)$. Observe that if $H_0$ is true, then $(X,Y)$ and $(X,\bar{Y})$ have the same joint distribution, while when $H_1$ holds, then $\bar{Y}$ has a ``different'' conditional distribution 
w.r.t.\ $X$ than $Y$. 
Our idea is to imitate sample $\CD_0= \{(X_i,Y_i)\}$ by generating alternative outputs for the original input points from the conditional distribution induced by the candidate function $f$, i.e., let $m>1$ be a user-chosen integer and let us defined samples
\begin{equation}
\begin{aligned}
\CD_j \doteq \{(X_i,\bar{Y}_{i,j})\}_{i=1}^n \;\;\;\; \text{for }\;\; j \in [m-1].
\end{aligned}
\end{equation}
An uninvolved way to produce $\bar{Y}_{i,j}$ for $i \in [n]$, $j \in [m-1]$ is as follows.  We generate i.i.d.\ uniform variables from $(-1,1)$. Let these  be $U_{i,j}$ for $i \in [n]$ and $j \in [m-1]$. Then we take
\begin{equation}
\bar{Y}_{i,j} = \BI (U_{i,j} \leq f(X_i))- \BI (U_{i,j} > f(X_i)),
\end{equation} 
for $(i,j) \in [n]\times[m-1]$.
The following remark highlights one of the main advantages of this scheme. 
\vspace{1.5mm}
\begin{remark}
If $H_0$ holds, then $\{(\CD_j)\}_{j=0}^{m-1}$ are conditionally i.i.d.\ w.r.t.\ $\{(X_i)\}_{i=1}^n$, hence they are also exchangeable.
\end{remark} 
\vspace{1.5mm}
The suggested {\color{sztakiblue} distribution-free} hypothesis tests are carried out via rank statistics as described in \cite{csajitamas2019}, where our resampling  
framework {\color{sztakiblue} for classification} was first introduced. {\color{sztakiblue} That is, we define a suitable ordering on the samples and accept the nullhypothesis when the rank of the original sample is not ``extremal'' (neither too low nor too high), i.e., the original sample does not differ significantly from the alternative ones.
More abstractly, we define our 
tests via ranking functions:}
\vspace{1.5mm}
\begin{definition}[ranking function]\label{def:ranking-function}
Let 
$\BA$ be a measurable space. A (measurable) function 
$\psi : \BA^m \to [\,m\,]$ is called a 
{\em ranking function} if for all $(a_1, \dots, a_m) \in \BA^m$ we have:%
\begin{enumerate}
\item[P1] For all permutations {\color{sztakiblue}$\nu$} of the set $\{2,\dots, m\}$, 
we have 
\begin{equation*}
\psi\big(\,a_1, a_{2}, \dots, a_{m}\,\big)\; = \;
\psi\big(\,a_1, a_{{\color{sztakiblue}\nu(2)}}, \dots, a_{{\color{sztakiblue}\nu(m)}}\,\big),
\end{equation*}
that is the function is invariant w.r.t. reordering the last $m-1$ terms of its arguments.
\item[P2] For all $i,j \in  [\,m\,]$,
if $a_i \neq a_j$, then we have
\begin{equation}
\psi\big(\,a_i, \{a_{k}\}_{k\neq i}\,\big)\, \neq \;\psi\big(\,a_j, \{a_{k}\}_{k\neq j}\,\big),
\end{equation}
where the simplified notation is justified by P1.
\end{enumerate}
\end{definition}
\vspace{1.5mm}
{\color{sztakiblue} Because of P2 when $a_1, \ldots, a_n \in \BA$  are pairwise different $\psi$ assigns a unique {\em rank} in $[m]$ to each $a_i$ by $\psi(a_i, \{a_k\}_{k\neq i}))$. We would like to consider the rank of $\CD_0$ w.r.t. $\{\CD_i\}_{I=1}^{m-1}$, hence  %
we apply ranking functions on $\CD_0,\ldots, \CD_{m-1}$.  One can observe} that these datasets are not necessarily pairwise different causing a technical challenge. To resolve  {\color{sztakiblue} ties in the ordering} we extend each sample with the different values of a uniformly generated (independently from every other variable) random permutation, $\pi$, on set $[m]$, i.e., we let
\vspace{-1mm}
\begin{equation}
\CD_j^\pi \doteq (\CD_j, \pi(j))\;\;\; \text{for}\;\;\; j = 1,\dots, m-1
\vspace{-1mm}
\end{equation}
{\color{sztakiblue}
and $\CD_0^\pi \doteq (\CD_0, \pi(m))$.
Assume that a ranking function  $\psi: (\BX \times \BY)^n \times [m] \to [m]$ is given. Then, we define our tests as follows. Let $p$ and $q$ be user-chosen integers such that $1 \leq p < q \leq m$. We accept hypothesis $H_0$ if and only if}
\begin{equation}
\begin{aligned}\label{eq:test}
p \leq \psi\big( \CD_0^\pi, \dots, \CD_{m-1}^{\pi}\big) \leq  q,
\end{aligned}
\end{equation}
{\color{sztakiblue} i.e., we reject the nullhypothesis if the computed rank statistic is ``extremal'' (too low or too high).}
Our main tool to determine exact type I error probabilities is Theorem \ref{theorem:8}, originally proposed in \cite[Theorem 1]{csajitamas2019}.
\vspace{1.5mm}
\begin{theorem}\label{theorem:8}
Assume that $\CD_0$ is an i.i.d.\ sample (A0). For all ranking function $\psi$, if $H_0$ holds true, then we have
\begin{equation}
\PP\, \Big( \,p \leq \psi\big( \CD_0^\pi, \dots, \CD_{m-1}^{\pi}\,\big) \leq  q\Big) = \frac{q-p+1}{m}.
\vspace{3mm}
\end{equation}
\end{theorem}
{\color{sztakiblue}
The intuition behind this result is that if $H_0$ holds true then the origininal dataset behaves similarly to the alternative ones, consequently, its rank in an ordering admits a (discrete) uniform distribution. The main power of this thoerem comes from its distribution-free and non-asymptotically guaranteed nature. 
Furthermore, we can observe that parameters $p$, $q$ and $m$ are user-chosen, hence the probability of the acceptance region can be controlled exactly when $H_0$ holds true, that is the probability of the type I error is exactly quantified.

The main statistical assumption of Theorem \ref{theorem:8} is very mild, that is we only require the data sample, $\CD_0$, to be i.i.d. Even though we presupposed that a ranking function is given, one can observe that our definition for $\psi$ is quite general. Indeed, it also allows some degenerate choices that only depend on the ancillary random permutation, $\pi$, which is attached to the datasets. Our intention is to exclude such options, therefore we examine the type II error probabilities of our hypothesis tests. We present two new ranking functions that are endowed with strong asymptotic bounds for their type II errors.
}

\vspace*{-1.8mm}
\subsection{Conditional Kernel Mean Embedding Based Tests}

The proposed ranking functions are defined via conditional kernel mean embeddings. The main idea is to compute the empirical estimate of the conditional kernel mean map based on all available samples, both the original and the alternatively generated ones, and compare these to an estimate obtained from the regression function of the nullhypothesis. The main observation is that the estimates based on the alternatively generated samples always converge to the theoretical conditional kernel mean map, which can be deduced from $f$, while the estimate based on the original sample, $\CD_0$, converges to the theoretical one only if $H_0$ holds true. We assume that:
\vspace{1mm}
\begin{itemize}
\item[A1] Kernel $l$ is the na\"ive kernel.\\[-3mm]
\item[A2] {\color{sztakiblue}Kernel $k$ is real-valued, measurable, $C_0$-universal \cite{carmeli2010vector} and bounded by  $C_k$ as well as $\Gamma = k\, \mathbf{Id}_\CG$.}%
\end{itemize}
\vspace{1mm}
{\color{sztakiblue}One can easily guarantee A2 by choosing a ``proper'' kernel $k$, e.g., a suitable choice is the Gaussian kernel, if $\BX = \RR^d$.}

%
We can observe that if a regression function, $f$, is given in binary classification, then the exact conditional kernel mean embedding $\mu_{\scriptscriptstyle{\bar{Y}| X}}$ and mean map $\mu_f$ can be obtained as
\begin{equation}\label{eq:cond_theoretic}
\begin{aligned}
&\mu_{\scriptscriptstyle{\bar{Y}| X}} = ( 1- p(X))l(\circ,-1) + p(X) l(\circ, 1) \;\; \text{and}\\[1mm]
&\mu_f(x) = ( 1- p(x))l(\circ,-1) + p(x) l(\circ, 1)\;\;\; P_{\scriptscriptstyle{X}}\text{-a.s.},
\end{aligned}
\end{equation}
because by the reproducing property for all $g \in \CG$ we have 
\begin{equation}
\begin{aligned}
&\langle( 1- p(X))l(\circ,-1) + p(X) l(\circ, 1),g\rangle_\CG \\[1mm]
&= ( 1- p(X))g(-1) + p(X) g(1) \myeq \EE \big[ g(\bar{Y}) \,|\, X \big].
\end{aligned}
\end{equation}
For simplicity, we denote $\mu_{f_*}$ by $\mu_*$.
We propose two methods to empirically estimate $\mu_f$. First, we use the regularized risk minimizer, $\widehat{\mu}$, defined in 
Theorem \ref{theorem:michelli}, and let
\begin{equation}\label{eq:r1}
\begin{aligned}
\widehat{\mu}_j^{(1)} \doteq \widehat{\mu}_{\CD_j} \;\;\;\text{for}\;\; j = 0, \dots,m-1.
\end{aligned}
\end{equation}
Second, we rely on the intuitive form of \eqref{eq:cond_theoretic} and estimate the conditional probability function $p$ directly by any standard method (e.g., $k$-nearest neighbors) and let 
\begin{equation}\label{eq:r2}
\begin{aligned}
\widehat{\mu}_j^{(2)}(x) \doteq (1- \widehat{p}_{j}(x))l(\circ,-1) + \widehat{p}_{j}(x) l(\circ, 1),
\end{aligned}
\end{equation}
for $j = 0, \dots,m-1$, where $\widehat{p}_j = \widehat{p}_{\CD_j}$ denotes the estimate of $p$ based on sample $\CD_j$.
The first approach follows our motivation by using a vector-valued RKHS and the {\color{sztakiblue} user-chosen 
kernel $\Gamma$ lets} us adaptively control the possibly high-dimensional scalar product. The second technique highly relies on the used conditional probability function estimator, hence we can make use of a broad range of point estimators available for this problem. {\color{sztakiblue} For brevity, we call the first approach vector-valued kernel test (VVKT) and the second approach point estimation based test (PET). The main advantage of VVKT comes from its nonparametric nature, while PET can be favorable when a priori information on the structure of $f$ is available.}

Let us define the ranking functions with the help of reference variables, which are estimates of the deviations between the empirical estimates and the theoretical conditional kernel mean map in some norm.
An intuitive norm to apply is the expected loss in $\norm{\cdot}_\CG$, i.e., for $\BX \to \CG$ type functions $\mu_f$, $\widehat{\mu} \in L_2(\PX;\CG)$ we consider the expected loss
\begin{equation} 
\begin{aligned}
\int_\BX \norm{ \mu_f(x) - \widehat{\mu}(x)}_\CG^2  \dd \PX(x).
\end{aligned}
\end{equation}
The usage of this ``metric'' is justified by \cite[Lemma 2.1]{park2020regularised}, where it is proved that for any estimator $\widehat{\mu}$ and conditional kernel mean map $\mu_f$, we have
\begin{equation}
\int_\BX \norm{ \mu_f(x) - \widehat{\mu}(x)}_\CG^2 \dd \PX (x)= \mathcal{E}_s(\widehat{\mu}) - \mathcal{E}_s(\mu_f),
\end{equation}
where the right hand side is the excess risk of $\widehat{\mu}$. The distribution of $X$ is unknown, thus the reference variables and the ranking functions are constructed as\footnote{{\color{sztakiblue}$Z_j^{(r)} \hspace*{-0.15mm}\prec_\pi \hspace*{-0.15mm} Z_0^{(r)} \Longleftrightarrow Z_j^{(r)}\hspace*{-0.15mm} < \hspace*{-0.15mm}Z_0^{(r)}$ or $\big(Z_j^{(r)}\hspace*{-0.15mm} =\hspace*{-0.15mm} Z_0^{(r)}$ and $\pi(j) \hspace*{-0.15mm}<\hspace*{-0.15mm} \pi(m)\big)$ }}
\begin{equation} \label{eq:squareroot}
	\begin{aligned}
		&Z_j^{(r)} \doteq \frac{1}{n}\sum_{i=1}^n \norm{ \mu_f(X_i) - \widehat{\mu}_j^{(r)}(X_i)}_\CG^2,\\
		&\CR_n^{(r)} \doteq 1 + \sum_{j=1}^{m-1} \BI\,\big(\, Z_j^{(r)} \prec_\pi Z_0^{(r)} \,\big)
	\end{aligned}
\end{equation}
for  $j = 0, \dots,m-1$ and  $r \in \{1,2\}$, {\color{sztakiblue} where $r$ refers to the two conditional kernel mean map estimators, \eqref{eq:r1} and \eqref{eq:r2}}. The acceptance regions of the proposed hypothesis tests are defined by \eqref{eq:test} with $\psi(\CD_0^\pi,\dots,\CD_{m-1}^\pi) = \CR_n^{(r)}$. {\color{sztakiblue} The idea is to reject $f$ when $Z_0^{(r)}$ is too high in which case our original estimate is far from the theoretical map given $f$. Hence setting $p$ to $1$ is favorable.} The main advantage of these hypothesis tests is that we can adjust the exact type I error probability to our needs for any finite $n$, irrespective of the sample distribution. Moreover, asymptotic guarantees can be ensured for the type II probabilities. We propose the following assumption to provide asymptotic guarantees for the tests:
{\color{sztakiblue}
\begin{itemize}
\vspace{1mm}
\item[B1] For the conditional kernel mean map estimates we have
\begin{equation*}
\frac{1}{n}\sum_{i=1}^n \norm{\mu_*(X_i) - \widehat{\mu}^{(1)}(X_i)}_\CG^2 \xrightarrow[n\to \infty]{\,a.s.\,} 0.
\end{equation*}
\end{itemize}That is we assume that the used regularized risk minimizer 
is consistent in the sense above. 
Condition B1, although nontrivial, is
however key in proving that a hypothesis test can preserve the favorable  asymptotic behaviour of the point estimator while also non-asymptotically guaranteeing a user-chosen type I error probability.}
{\color{sztakiblue}
\begin{theorem}\label{theorem:alg2}
Assume that A0, A1, A2 and $H_0$ hold true, then for all sample size $n \in \NN$ we have
\begin{equation}
\PP \Big(\, \CR_n^{(1)} \leq 	q \,\Big) = \frac{q}{m}.
\end{equation} 
If A0, A1, A2, B1, $q < m$ and $H_1$ hold, then
\begin{equation}\label{equation:thm4}
\PP\bigg(\,\bigcap_{N=1}^\infty \bigcup_{n=1}^N \big\{\CR_n^{(1)} \leq q \big\} \,\bigg) = 0.
\vspace{2mm}
\end{equation}
\end{theorem}}
\vspace{1mm}
{\color{sztakiblue}
The tail event in \eqref{equation:thm4} is often called the ``$\limsup$'' of events $\big\{\CR_n^{(1)} \leq q \big\}$, where $n \in \NN$. In other words, the theorem states that
the probability of type I error is exactly $1-q/m$.
Moreover, under $H_1$, $\CR_n^{(1)} \leq q$ happens infinitely many times with zero probability, equivalently a ``false'' regression function is (a.s.) accepted at most finitely many times. The pointwise convergence of the type II error probabilities to zero (as $n \to \infty$) is a straightforward consequence of \eqref{equation:thm4}.}

{\color{sztakiblue}
A similar theorem holds for the second approach, where we assume the consistency of $\widehat{p}$ in the following sense:
\begin{itemize}
\item[C1] For conditional probability function estimator $\widehat{p}$ we have
\begin{equation}
\frac{1}{n}\sum_{i=1}^n(p(X_i) - \widehat{p}(X_i))^2 \xrightarrow[n\to \infty]{\,a.s.\,}0.
\end{equation}
\end{itemize}
Condition C1 holds for a broad range of conditional probability estimators (e.g., kNN, various kernel estimates), however most of these techniques 
make stronger assumptions on the data generating distribution than we do. As in Theorem \ref{theorem:alg2} the presented stochastic guarantee for the type I error is non-asymptotic, while for the type II error it is asymptotic.
\vspace{1.5mm}
\begin{theorem}\label{theorem:alg3}
Assume that A0, A1, A2 and $H_0$ hold true, then for all sample size $n \in \NN$ we have
\begin{equation}
\PP\Big(\, \CR_n^{(2)} \leq 	q \,\Big) = \frac{q}{m}.
\end{equation} 
If A1, C1, $q < m$ and $H_1$ hold, then 
\begin{equation}\label{eq:limsup_3}
\PP\bigg(\,\bigcap_{N=1}^\infty \bigcup_{n=1}^N \big\{\CR_n^{(2)} \leq q \big\} \,\bigg) = 0.
\vspace{3mm}
\end{equation}
\end{theorem}
\noindent The proof of both theorems can be found in the appendix. 
}

\begin{figure}[tb]
	\vspace*{-2mm}
    \centering
	\hspace*{-1mm}	
	\subfigure[VVKT]{\label{fig:covar}\includegraphics[height=40mm, width = 40mm]{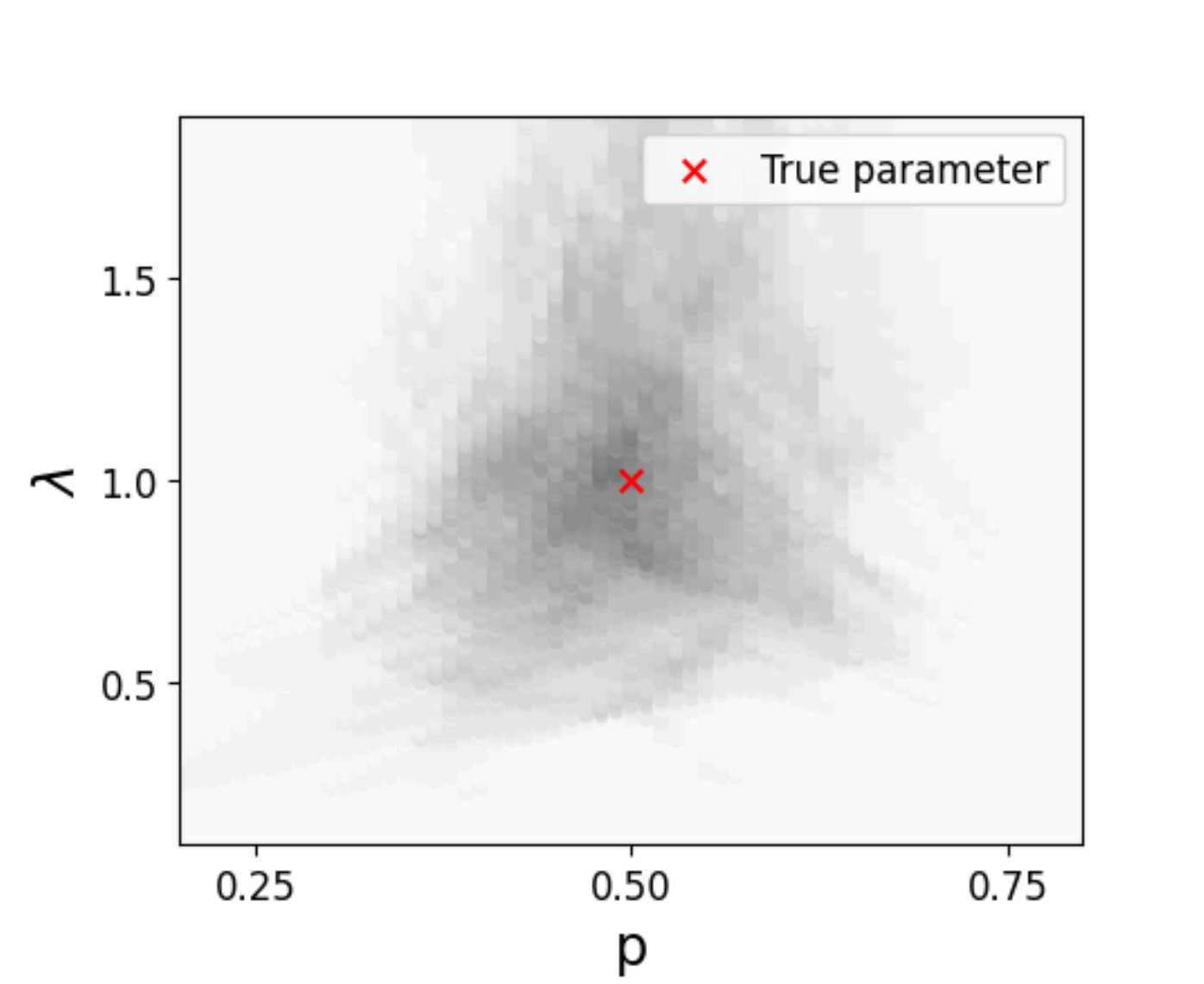}}
 	\subfigure[PET]{\label{fig:cond_emb}\includegraphics[height=40mm, width = 48mm]{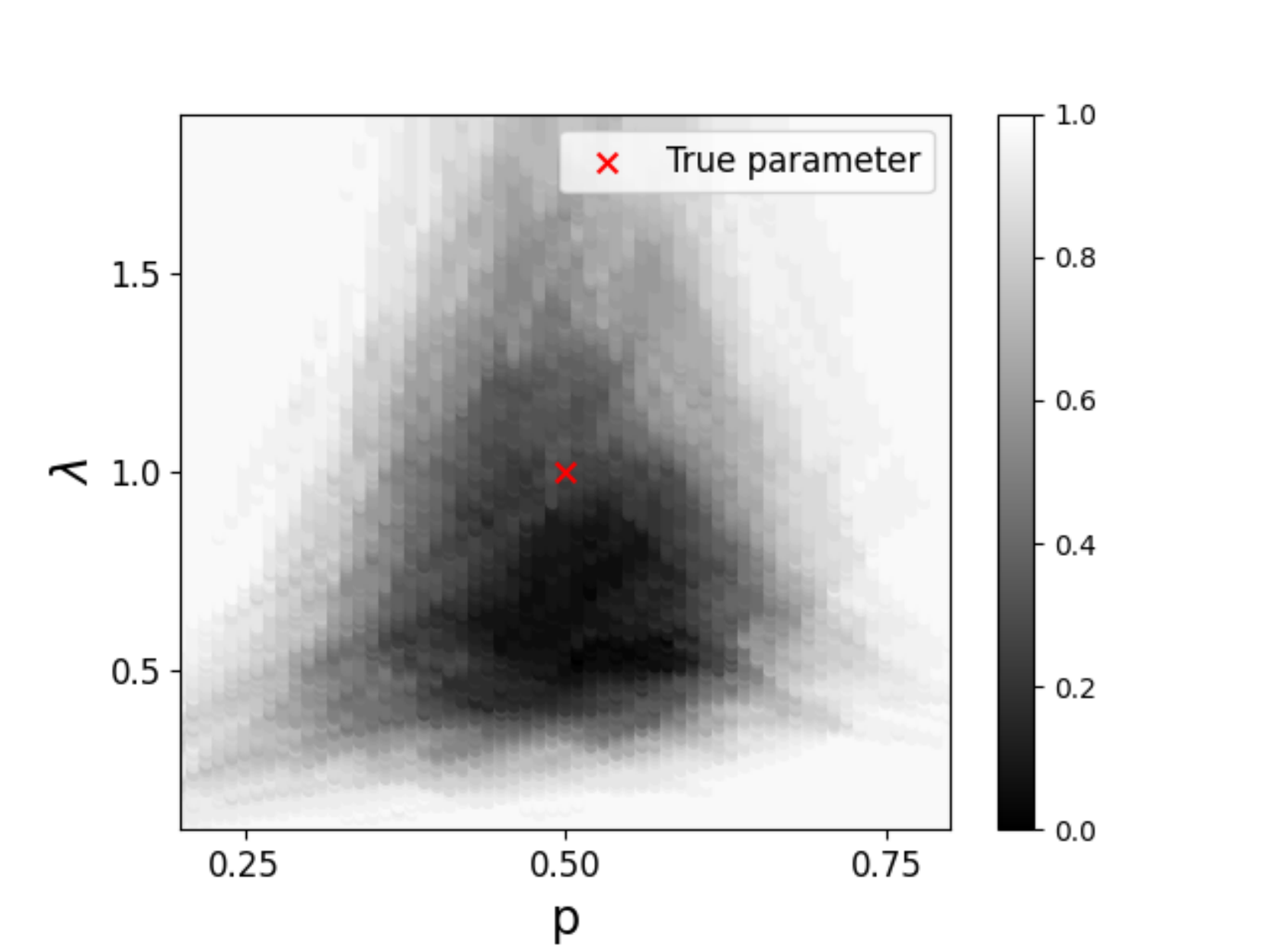} 	\hspace*{-6mm}}
 	\caption{The normalized ranks of VVKTs and PETs are represented with the darkness of points on a grid based on a sample of size $n=50$ with resampling parameter $m=40$. For VVKTs we used a Gaussian kernel with $\sigma =\nicefrac{1}{2}$. For PETs we used kNN with $k=\lfloor \sqrt{n}\rfloor$ neighbors.}
	\subfigure[Ranks $(p=0.3$, $\lambda=1.3)$]{\label{fig:consistency_1}\includegraphics[height=40mm, width = 40mm]{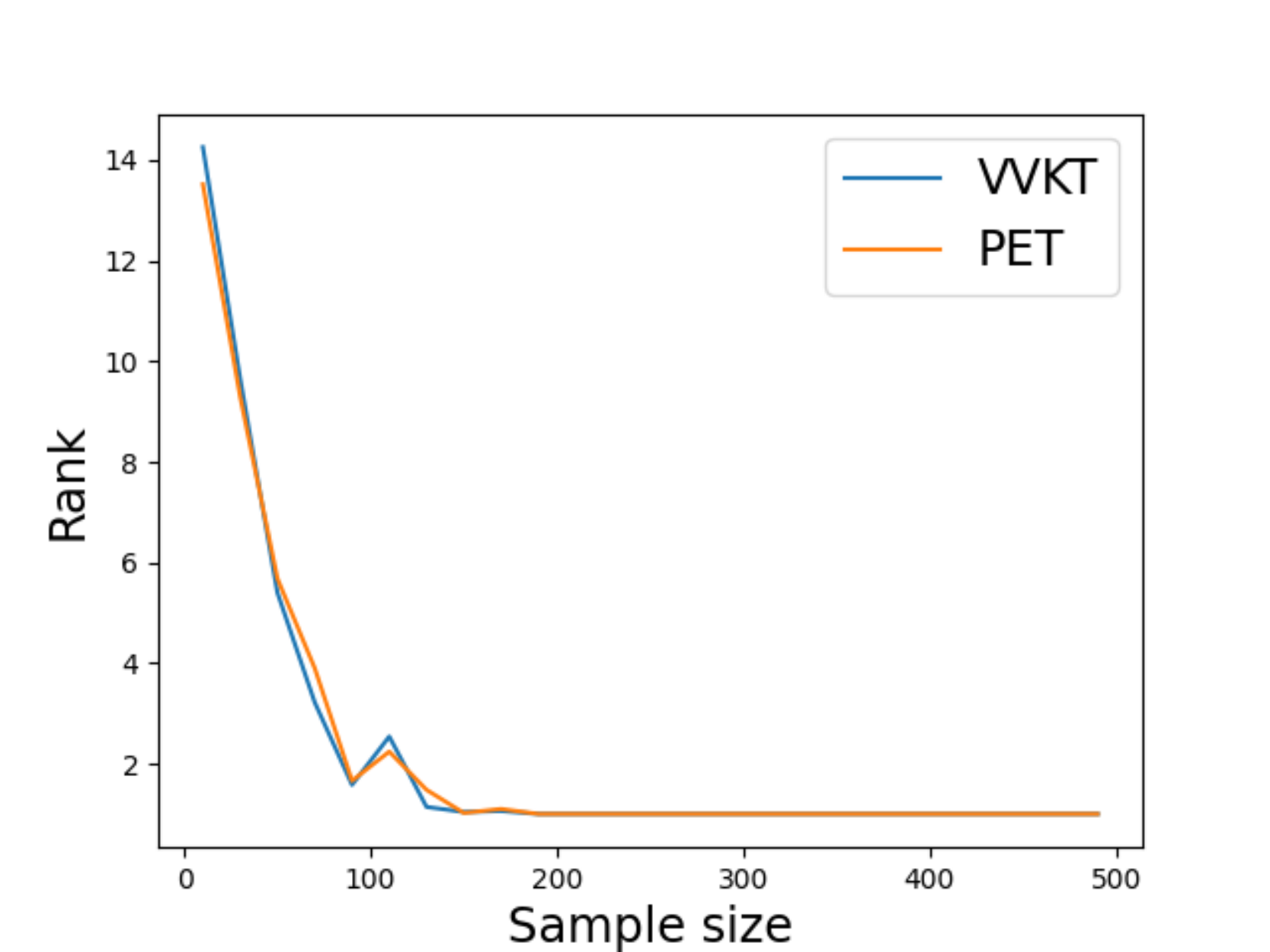}} 
	\subfigure[Ranks $(p=0.4$, $\lambda=1.2)$]{\label{fig:consistency_2}\includegraphics[height=40mm, width = 40mm]{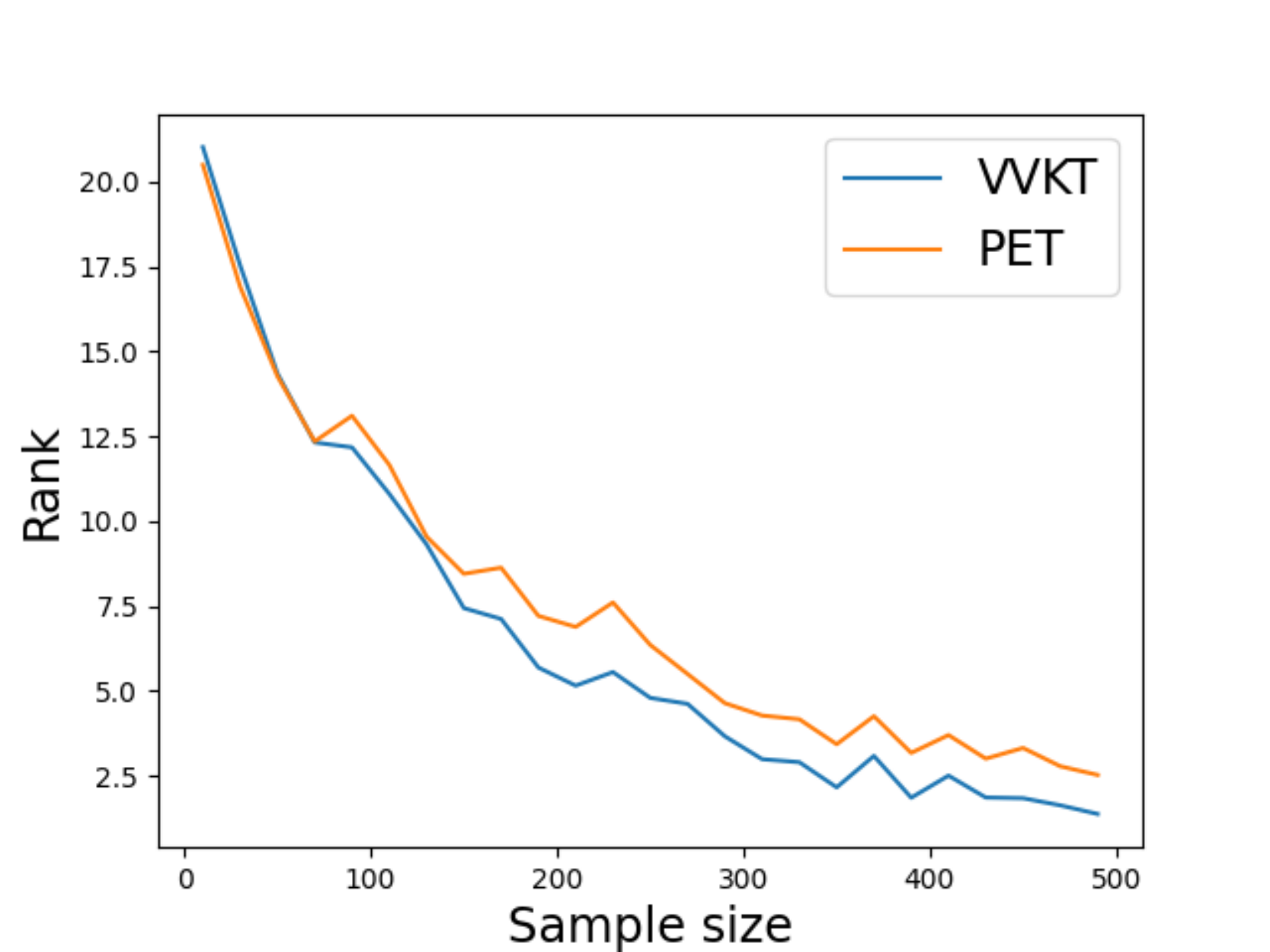}} 	
    \caption{The ranks of the reference variables are shown as functions of the sample size for two arbitrarily chosen ``false'' candidate functions.}
\label{fig:experiments}
\vspace*{-3mm}
\end{figure} 

\vspace*{0mm}
\section{Numerical Simulations}

We made numerical experiments on a synthetic dataset to illustrate the suggested hypothesis tests. We considered the one dimensional, bounded input space $[-1,1]$ with binary outputs. The marginal distribution of the inputs were uniform. The true regression function was the following model:
\begin{equation}
f_*(x) \doteq \frac{p_*\cdot e^{-(x-\mu_1)^2/\lambda_*}- (1-p_*)e^{-(x-\mu_2)^2/\lambda_*}}{p_*\cdot e^{-(x-\mu_1)^2/\lambda_*}+ (1-p_*)e^{-(x-\mu_2)^2/\lambda_*}},
\end{equation}
where $p_*=0.5$, $\lambda_*=1$, $\mu_1 = 1$ and $\mu_2 =-1$. This form of the regression function is the reparametrization of a logistic regression model, which is an archetypical approach for binary classification. We get the same formula if we mix together two Gaussian classes.
The translation parameters ($\mu_1$ and $\mu_2$) were considered to be known to illustrate the hypothesis tests with two dimensional pictures. The sample size was $n =50$ and the resampling parameter was $m=40$. We tested parameter pairs of $(p,\lambda)$ on a fine grid with stepsize $0.01$ on $[0.2,0.8]\times [0.5,1.5]$.  
The two hypothesis tests are illustrated with the generated {\color{sztakiblue}rank statistics} for all tested parameters on Figures \ref{fig:covar} and \ref{fig:cond_emb}. {\color{sztakiblue} These normalized values} are indicated with the colors of the points.  Kernel $k$ was a Gaussian with parameter $\sigma = \nicefrac{1}{2}$ for the VVKTs and we used kNN-estimates for PETs with $k=\lfloor \sqrt{n}\rfloor$ neighbors. We illustrated the consistency of our algorithm by plotting the ranks of the reference variables for parameters $(0.3,1.3)$ and  $(0.4,1.2)$ for various sample sizes in Figures \ref{fig:consistency_1} and \ref{fig:consistency_2}. We took the average rank over several runs for each datasize. We can see that the reference variable corresponding to the original sample rapidly tends to be the greatest, though the rate of convergence depends on the 
particular hypothesis.

\vspace*{-1mm}
\section{Conclusions}

In this paper we have introduced two new distribution-free hypothesis tests for the regression functions of binary classification based on conditional kernel mean embeddings. 

Both proposed methods incorporate the idea that the output labels can be resampled based on the candidate regression function we are testing. The main advantages of the suggested hypothesis tests are as follows: (1) they have a user-chosen exact probability for the type I error, which is non-asymptotically guaranteed for any sample size; furthermore, (2) they are also consistent, i.e., the probability of their type II error converges to zero, as the sample size tends to infinity. 
Our approach can be used to quantify the uncertainty of classification models and it can form a basis of confidence region constructions.
\vspace*{-1mm}
\bibliographystyle{ieeetr}
\bibliography{embedding_tests_bib} 

\appendix
\label{app:theorem}

\subsection{Proof of Theorem \ref{theorem:alg2}}\label{app:a}

\noindent
\begin{proof}
The first equality follows from Theorem \ref{theorem:8}. When the alternative hypothesis holds true, i.e., $f \neq f_*$, 
let $S_n^{(j)} = \sqrt{Z_j^{(1)}}$ for $j= 0, \dots, m-1$. %
It is sufficient to show that $S_n^{(0)}$ tends to be the greatest in the ordering as $n \to \infty$, because the square root function is order-preserving. For $j=0$ by the reverse triangle inequality we have
\begin{equation}\label{eq:pr2}
\begin{aligned}
&S_n^{(0)}  = \sqrt{\frac{1}{n} \sum_{i=1}^n \norm{\mu_f(X_i) - \widehat{\mu}_0^{(1)}(X_i)}_\CG^2} \\[2mm]
&\geq  \sqrt{\frac{1}{n} \sum_{i=1}^n \norm{\mu_f(X_i) - \mu_*(X_i)}_\CG^2} \\[2mm]
&-\sqrt{\frac{1}{n} \sum_{i=1}^n \norm{ \mu_*(X_i)-\widehat{\mu}_0^{(1)}(X_i)}_\CG^2}.
\end{aligned} 
\end{equation}
The first term converges to a positive number, as
\vspace{-1mm}
\begin{equation*}
\hspace*{-3mm}
\frac{1}{n} \sum_{i=1}^n \norm{\mu_f(X_i) - \mu_*(X_i)}_\CG^2 =
\end{equation*}
\begin{equation}
\hspace*{-3mm}
\begin{aligned}
&=  \frac{1}{n} \sum_{i=1}^n \|(p(X_i) - p_*(X_i))l(\circ,1) \\
&\hspace{2cm}+ (p_*(X_i)- p(X_i))l(\circ,-1) \|_\CG^2\\
&=  \frac{1}{n} \sum_{i=1}^n \big[\big(p(X_i) - p_*(X_i)\big)^2 l(1,1) \\
&\hspace{2cm}+ \big(p_*(X_i)- p(X_i)\big)^2l(-1,-1)\big]\\
&= \frac{2}{n} \sum_{i=1}^n \Big[\big(p(X_i) - p_*(X_i)\big)^2\Big] \to 2\, \EE \Big[\big(p(X)- p_*(X)\big)^2\Big],\hspace*{-10mm}
\end{aligned}
\vspace{-1mm}
\end{equation}
where we used the SLLN and that $l(1,-1)=l(-1,1)=0$. When $f \neq f_*$ we have  that $\kappa \doteq \EE \Big[\big(p(X)- p_*(X)\big)^2\Big] >0$.
{\color{sztakiblue}
The second term almost surely converges to zero by B1, hence we can conclude that $S_n^{(0)} \xrightarrow{\,a.s.\,}\sqrt{2\kappa}$.}

For $j \in [m-1]$ 
variable 
$Z_j^{(1)}$ has a similar form as the second term in \eqref{eq:pr2}, thus its 
(a.s.) convergence to $0$ {\color{sztakiblue} follows from B1}, i.e., we get {\color{sztakiblue}$Z_j^{(1)} \xrightarrow{\,a.s.\,}0$}. Hence, $Z_0^{(1)}$ {\color{sztakiblue}(a.s.)} tends to become the greatest, implying \eqref{equation:thm4} for $q < m$.
\end{proof}

\vspace*{-2mm}
\subsection{Proof of Theorem \ref{theorem:alg3}}

\noindent
\begin{proof}
The first part of the theorem follows from Theorem \ref{theorem:8} with $p=1$.
For the second part let $f \neq f_*$. We transform the reference variables as
\begin{equation}
\begin{aligned}
&Z_j^{(2)} = \frac{1}{n} \sum_{i=1}^n \|(p(X_i) l(\circ,1) + (1- p(X_i))l(\circ,-1)\\
& \hspace{2cm} - \big(\,\widehat{p}_j(X_i) l(\circ,1) + (1- \widehat{p}_j(X_i))l(\circ,-1)\big)\|_\CG^2\\
&=  \frac{1}{n} \sum_{i=1}^n \|(p(X_i)-\widehat{p}_j(X_i)) l(\circ,1) \\
& \hspace{2cm}+ (\,\widehat{p}_j(X_i)- p(X_i))l(\circ,-1) \|_\CG^2\\
\end{aligned}
\end{equation}
\begin{equation*}
\begin{aligned}
&= \frac{1}{n} \sum_{i=1}^n \big( (p(X_i)-\widehat{p}_j(X_i))^2 l(1,1) \\
&\hspace{2cm}+   (\widehat{p}_j(X_i)- p(X_i))^2l(-1,-1)\big) \\
&= \frac{2}{n}\sum_{i=1}^n (p(X_i)-\widehat{p}_j(X_i))^2 
\end{aligned}
\end{equation*}
for $j =0,\dots, m-1$. From C1 it follows that $Z_j^{(2)}$ goes to zero a.s.\ for $j \in [m-1]$. For $j=0$ we argue that $Z_0^{(2)} \to \kappa$ for some $\kappa >0$. Notice that
\begin{equation*}
\hspace*{-2mm}
\begin{aligned}
&\frac{2}{n}\sum_{i=1}^n (p(X_i)-\widehat{p}_{0}(X_i))^2 \\
& = \frac{2}{n}\sum_{i=1}^n (p(X_i)- p_*(X_i))^2+ \frac{2}{n}\sum_{i=1}^n( p_*(X_i) - \widehat{p}_{0}(X_i))^2 \\
&+ \frac{4}{n} \sum_{i=1}^n \Big((\,p(X_i) - p_*(X_i)\,)(\,p_*(X_i) - \widehat{p}_{0}(X_i)\,) \Big) \hspace*{-10mm}
\end{aligned}
\end{equation*}
holds.  By the SLLN the first term converges to a positive number, $\EE \big[( p_*(X)- p_0(X))^2\big] >0$. By C1 the second term converges to $0$ (a.s.). The third term also tends to $0$ by the Cauchy-Schwartz inequality and C1, as for $x \in \BX$ we have $|\,p(x)- p_*(x)\,| \leq  1$.
We conclude that  if $f \neq f_*,$ $Z_0^{(2)}$ (a.s.) tends to be the greatest in the ordering implying \eqref{eq:limsup_3}.
\end{proof}

\end{document}